\documentclass{article}
\usepackage[utf8]{inputenc}
\usepackage{amsmath,amsthm, amssymb, amsfonts}
\usepackage{fullpage}
\usepackage[pdfencoding=auto]{hyperref}
\usepackage{graphicx}
\usepackage{wasysym}
\usepackage{todonotes}
\usepackage{comment}
\usepackage{mathtools}
\usepackage{multicol}
\usepackage{wrapfig}
\usepackage{microtype}
\usepackage[]{algorithm2e}
\usepackage{times}
\usepackage[ngerman, english]{babel}

\definecolor{quotationcolour}{HTML}{F0F0F0}
\definecolor{quotationmarkcolour}{HTML}{1F3F81}

\usepackage{xcolor}

\newcommand\blfootnote[1]{%
  \begingroup
  \renewcommand\thefootnote{}\footnote{#1}%
  \addtocounter{footnote}{-1}%
  \endgroup
}

\allowdisplaybreaks[1]
\usepackage{sistyle} 

\newcommand{\N}{\mathbb{N}}

\newcommand{\R}{\mathbb{R}}

\let\emptyset\varnothing

\newtheorem{theorem}{Theorem}[section]
\newtheorem*{theorem*}{Theorem}

\newtheorem{definition}[theorem]{Definition}

\newtheorem*{remark*}{Remark}
\newtheorem*{proposition*}{Proposition}

\numberwithin{equation}{section}

\DeclareMathOperator{\suppp}{supp \,}

\definecolor{darkcandyapplered}{rgb}{0.64, 0.0, 0.0}

\title{The Oracle of DLphi}

\author{Dominik Alfke \\ University of Chemnitz \and  
Weston Baines \\ Texas A\&M \and  
Jan Blechschmidt \\ University of Chemnitz  \and 
Mauricio J. del Razo Sarmina \\ FU Berlin \and
Amnon Drory \\ Tel Aviv University \and
Dennis Elbrächter \\ University of Vienna \and
Nando Farchmin \\ Physikalisch-Technische Bundesanstalt \and
Matteo Gambara \\ ETH Z\"urich \and
Silke Glas \\ University of Ulm \and
Philipp Grohs \\ University of Vienna \and
Peter Hinz \\ ETH Z\"urich \and
Danijel Kivaranovic \\ University of Vienna \and   
Christian Kümmerle \\ TU Munich\and 
Gitta Kutyniok \\ TU Berlin \and 
Sebastian Lunz \\ University of Cambridge \and  
Jan Macdonald \\ TU Berlin \and  
Ryan Malthaner \\ Texas A\&M \and 
Gregory Naisat \\ University of Chicago \and  
Ariel Neufeld \\ ETH Z\"urich \and  
Philipp Christian Petersen \\ University of Oxford \and 
Rafael Reisenhofer \\ University of Vienna \and
Jun-Da Sheng \\ UC Davis\and  
Laura Thesing \\ University of Cambridge \and  
Philipp Trunschke \\ TU Berlin \and
Johannes von Lindheim\\ TU Berlin \and 
David Weber \\ UC Davis \and
Melanie Weber \\ Princeton University \blfootnote{All authors are aware of this work. }}

\date{}

\begin{document}
\maketitle

\begin{abstract}
We present a novel technique based on deep learning and set theory which yields exceptional classification and prediction results. Having access to sufficiently large amount of labelled training data, our methodology is capable of predicting the labels of the test data almost always even if the training data is entirely unrelated to the test data. In other words, we prove in a specific setting that as long as one has access to enough data points, the quality of the data is irrelevant.
\end{abstract}


\section{Introduction}

This paper takes aim at achieving nothing less than the impossible. To be more precise, we seek to predict labels of unknown data from entirely uncorrelated labelled training data. This will be accomplished by an application of an algorithm based on deep learning, as well as, by invoking one of the most fundamental concepts of set theory.  

Estimating the behaviour of a system in unknown situations is one of the central problems of humanity. Indeed, we are constantly trying to produce predictions for future events to be able to prepare ourselves. For example, the benefits of accurate weather forecast include allowing us to make the right choice of clothing for the day, deciding if we should go by bike or take the bus to work, or if we should seek shelter from a natural catastrophe \cite{brown2008famine}. Election predictions give companies the possibility of funding the most promising candidates and gain in influence. Analysing the history of class conflicts allows making reliable predictions of our future economic systems \cite{marx1867kapital}.
Finally, precise knowledge of the behaviour of the stock market enables banks to make profits and strategically destabilise the economy \cite{IMFReport}. 

One of the most promising and successful techniques for such predictions are data-driven methods, most importantly \emph{deep learning} (DL). Nonetheless, it has been conventional wisdom that not even deep learning techniques can predict labels of unseen data points if they have nothing to do with the training data. In this work, we challenge this paradigm and show that there is no need to use training data that is particularly correlated with the data points that should be predicted. Indeed, we will see below, that as long as one has a sufficient amount of training data, one can accurately predict \emph{everything}, even events that are entirely independent of every data point used for training! In other words, not the quality of the training data is important but its sheer amount.

\subsection{Related work}

We will recall a couple of relevant articles that highlight the current efficiency of deep neural networks and then cite a number of our articles for no apparent reason. Neural networks were initially introduced in the 1940s by McCulloch and Pitts \cite{MP43} in an attempt to mathematically model the human brain. Later, this framework was identified as a flexible and powerful computational architecture which then led to the field of deep learning \cite{Goodfellow-et-al-2016, LeCun2015DeepLearning, schmidhuber2015deep}. Deep learning, roughly speaking, deals with the data-driven manipulation of neural networks. These methods turned out to be highly efficient to the extent that deep learning based methods are state-of-the-art technology in all image classification tasks 
\cite{Huang2017DenselyCC, simonyan2014very, Krizhevsky2012Imagenet}. They have revolutionised the field of speech recognition \cite{hinton2012deep, dahl2012context, wu2016stimulated} and have achieved a level of skill in playing games that humans or the best alternative algorithms cannot match anymore \cite{silver2017mastering, usunier2016episodic, yannakakis2017artificial}.

Many mathematicians have since been trying to understand why these algorithms are superior to the classical approaches. It was observed that deep networks naturally produce invariances, which might help to explain their excellent behaviour as a high-dimensional classifier \cite{bruna2013invariant, wiatowski2018mathematical}. The learning procedure is very unlikely to end up in bad local minima as has been observed in \cite{vidal2017mathematics, kawaguchi2016deep}. Alternative approaches focus on approximation theory; we mention \cite{Cybenko1989, Hornik1989universalApprox, Barron1993,Mhaskar1993, mhaskar-micchelli,petersen2018optimal,ShaCC2015provableAppDNN,YAROTSKY2017103} which is undoubtedly the second most biased selection of references we are about to make in this article.

A couple of additional papers that the reader should be aware of are \cite{alfke2018nfft, grohs2014parabolic, petersen2018optimal, drory2014semi, del2016numerical,perekrestenko2018universal,lunz2018adversarial, reisenhofer2018haar, saito2017underwater, neufeld2013superreplication, hansen2018stable, glas2017two, hinz2018framework, blechschmidt2018improving, macdonald2018improved, weber2017coarse, monshi2016burden}. If not for their relevance to deep learning, these articles at least give the reader an impression of the general interests of the authors.

\subsection{Notation}
For a set $\Omega$, we denote by $\mathcal{P}(\Omega)$ the power set of $\Omega$. In the following we will restrict out attention to the space $[0,1]^d$ for $d \in \N$ and unless specified differently we equip it with the Lebesgue measure, which in this case is a probability measure. For $1\leq p < \infty$, we denote by $L^p([0,1]^d)$ the classical Lebesgue spaces. Finally, we denote by ${[0,1]}^{([0,1]^d)}$ the set of all indexed sets $(y_i)_{i \in [0,1]^d}$ where $y_i \in [0,1]$ or equivalently the set of all maps from $[0,1]^d$ to $[0,1]$.

\section{Neural Networks}

We start by introducing neural networks as functions that can be written as an alternating application of affine linear maps and an activation function. 

\begin{definition} 
For $d, L \in \N$, a \emph{neural network with activation function $\varrho: \R \to \R$} is a function $\Phi: [0,1]^d \to \R$ such that there exist $N_0, \dots, N_L \in \N$, where $N_0 \coloneqq d$ and $N_L \coloneqq 1$ as well as affine linear maps $T_{\ell}: \R^{N_{\ell-1}} \to \R^{N_{\ell}}$ such that 
$$
\Phi(x) = T_L( \varrho(T_{L-1}( \dots \varrho(T_1(x))))), \text{ for all } x \in [0,1]^d,
$$
where $\varrho$ is applied coordinate-wise. The set of all neural networks which have a representation as above is denoted by
$\mathcal{NN}_{d,\varrho, L}$. Additionally, we denote
$$
\mathcal{NN}_{d,\varrho} \coloneqq \bigcup_{L \in \N}\mathcal{NN}_{d,\varrho, L}.
$$
\end{definition}
We recall one of the main results in neural network theory: the universal approximation theorem. The original statements are due to Hornik \cite{HORNIK1991251} and Cybenko \cite{Cybenko1989}, but we instead state a more general version describing approximation with respect to $L^p$ norms.
\begin{theorem}[\cite{leshno1993multilayer}]\label{thm:UniversalApprox}
Let $d\in \N$, $\varrho: \R \to \R$ be continuous and assume that $\varrho$ is not a polynomial. Then, for all $L \geq 2$ and all $1 \leq p < \infty$,
$\mathcal{NN}_{d,\varrho, L}$ is dense in $L^p([0,1]^d)$. In particular, $\mathcal{NN}_{d,\varrho}$ is dense in $L^p([0,1]^d)$, for $1\leq p< \infty$.
\end{theorem}

\section{Main Results}

As already announced in the introduction, we aim to construct a prediction algorithm that can produce a correct classification of unseen data based on uncorrelated training data. This formulation is sufficiently vague to be considerably impressive; however, it lacks the necessary rigour to be considered a mathematical statement. To make the statement more precise, we introduce the following definitions. 

\begin{definition}\label{randomsets}
Let $\mu$ be the Lebesgue measure on $[0,1]^d$ (which corresponds to the uniform probability distribution on $[0,1]^d$) and let $\nu$ be any discrete probability distribution on $\N$. We call a set $X$ obtained as the union of $k\in\N$ samples $X_1,\dots,X_k$ drawn i.i.d. according to $\mu$ a \emph{random set of size at most $k$}. We call a set $X$ generated by the following sampling procedure a \emph{random set of finite size (or simply a finite random set)}: First draw $k\in\N$ according to $\nu$, then generate a random set of size at most $k$ as before.
\end{definition}

Note that for the uniform distribution $\mu$ we will almost surely generate a set of size exactly $k$ (as opposed to at most $k$) by drawing $k$ i.i.d. samples for every $k\in\N$. It is not hard to see that for any finite $S\subset[0,1]^d$ and any random finite set $X\subset[0,1]^d$ the probability of $S$ and $X$ having non-empty intersection is zero. This is independent of $\nu$ as long as the uniform measure $\mu$ is used for the sampling. 

\begin{definition} \label{def:PrPrPr}
Let $d\in \N$. A map $L: \mathcal{P}([0,1]^d) \times {[0,1]}^{([0,1]^d)} \to {[0,1]}^{([0,1]^d)}$ is said to possess the \emph{precise prediction property} if for every set of labelled data $(y_i)_{i \in [0,1]^d} \in {[0,1]}^{([0,1]^d)}$ it holds that, for almost all finite random sets $X\subset [0,1]^d$, 
$$
L\left(X, \left(y_i^X\right)_{i \in [0,1]^d} \right) =  \left({y}_i\right)_{i \in [0,1]^d}, 
$$ 
where $y_i^X \coloneqq y_i$ if $i \not \in X$ and $y_i^X \coloneqq 0$ if $i \in X$.
\end{definition}
An honest description of Definition \ref{def:PrPrPr} is that given a ground truth $(y_i)_{i\in [0,1]^d}$ of labelled data, possibly without any inherent structure, there exists a learning algorithm that, when presented with a large subset of the labelled data (namely with $(y_i)_{i\in [0,1]^d\setminus X}$ for a generic finite $X\subset [0,1]^d$), can predict the missing data almost surely.

It turns out that there is an algorithm that, from the masked data $(y_i^X)_{i\in [0,1]^d}$, produces a deep neural network predicting the labels $(y_i)_{i\in X}$ almost surely. This deep network can be the basis of a map that possesses the precise prediction property. We shall present this result in the following theorem. 
What is more is that the proof below is semi-constructive. All that is required to perform the construction of a the deep neural network is an unwavering belief in the axioms of set theory, a potentially significant amount of book-keeping, and the solution of a moderate regression problem. 
Naturally, conventional wisdom suggests that the explicit construction of the proof can be replaced by training a sufficiently large and deep neural network via stochastic gradient descent.
\begin{theorem}\label{thm:main} 
Let $\varrho: \R \to \R:$ $\varrho(x) \coloneqq \max\{0,x\}$. Then, for all $\epsilon>0$, there exists a map $\mathcal{L}: \mathcal{P}([0,1]^d) \times {[0,1]}^{([0,1]^d)} \to \mathcal{NN}_{d, \varrho}$ such that: For every set of labelled data $(y_i)_{i \in [0,1]^d} \in {[0,1]}^{([0,1]^d)}$ it holds that, for almost all finite random sets $X\subset [0,1]^d$, 
\begin{align}\label{eq:PerfectAccuracyOnTestData}
\mathcal{L}\left(X, \left(y_i^X\right)_{i \in [0,1]^d} \right) (j) =  y_j, \text{ for all } j \in X.  
\end{align}
If, additionally, the map 
$$
f: [0,1]^d \to [0,1]: j \mapsto y_j
$$
is integrable, then we have that 
\begin{align}\label{eq:goodAccuracyOnTrainingData}
\int_{[0,1]^d} \left|\mathcal{L}\left(X, (y_i^X)_{i \in [0,1]^d} \right) (j) - f(j)\right| \, dj < \epsilon.
\end{align}
In particular, the map 
\begin{align*}
L \colon \mathcal{P}\left([0,1]^d\right) \times {[0,1]}^{([0,1]^d)} &\to  {[0,1]}^{([0,1]^d)}\\
L(X, (y_i)_{i \in [0,1]^d}) (j) & \coloneqq \left\{ \begin{array} {ll}
\mathcal{L}\left(X, \left(y_i^X\right)_{i \in [0,1]^d} \right) (j) &\text{if } j \in X\\
y_j &\text{else}
\end{array}\right.
\end{align*}
has the precise prediction property.
\end{theorem}
\begin{proof}
We define an equivalence relation $\sim$ on $ {[0,1]}^{([0,1]^d)} $ by defining $(y_i)_{i \in [0,1]^d} \sim (z_i)_{i \in [0,1]^d}$ if $y_i = z_i$ for all but finitely many $i \in [0,1]^d$. We denote the equivalence class containing $(y_i)_{i \in [0,1]^d}$ by $[[(y_i)_{i \in [0,1]^d}]]$. 
By the axiom of choice, there exists a function $\lambda$ that maps every equivalence class $\Xi \subset {[0,1]}^{([0,1]^d)}$ to an associated representative $(y_i)_{i \in [0,1]^d} \in \Xi$, i.e., 
$$
\lambda\left([[(y_i)_{i \in [0,1]^d}]]\right) \sim (y_i)_{i \in [0,1]^d}
$$
for all $(y_i)_{i \in [0,1]^d} \in {[0,1]}^{([0,1]^d)}$.

Now let $\epsilon>0$. Let $X \subset [0,1]^d$ (not necessarily finite) and let $(y_i)_{i \in [0,1]^d} \in {[0,1]}^{([0,1]^d)}$. Let $(z_i)_{i \in [0,1]^d} \coloneqq \lambda([[(y_i)_{i \in [0,1]^d}]])$ be the representative of the equivalence class $[[(y_i)_{i \in [0,1]^d}]]$. We proceed by defining several neural networks and consider the case of infinite and finite $X$ separately.

First, in the case of infinite $X$ we simply set $\Phi_{\epsilon, X} = 0$. 

Second, in the case of finite $X$ let
\[ 
g \colon [0,1]^d \to [0,1]\colon i \mapsto z_i.
\]
If $g$ is measurable and integrable then by Theorem \ref{thm:UniversalApprox} there exists a neural network $\Phi_\epsilon$ such that 
\[
\int_{[0,1]^d}|g(j) - \Phi_\epsilon(j)| \, dj < \frac{\epsilon}{2}.
\]
Otherwise, we set $\Phi_\epsilon = 0$. Now further, for every  $k \in X$ set $r_k \coloneqq g(k) - \Phi_\epsilon(k)$ and define
\begin{align}
\Phi_{k,n} : [0,1]^d &\to \R: \nonumber\\ 
j &\mapsto r_k \varrho\left( \sum_{\ell = 1}^d  \varrho\left(n \left(j_\ell - k_\ell+ \frac{1}{n} \right) \right) - 2 \varrho\left(n \big(j_\ell - k_\ell\big)\right) +  \varrho\left(n \left(j_\ell - k_\ell - \frac{1}{n}\right)\right)   - \left(d-1\right) \right ),\label{eq:Phin}
\end{align}
where $j_\ell$ and $k_\ell$ denote the $\ell$-th coordinate of $j$ and $k$ respectively. It is not hard to see that $\Phi_{k,n}(k) = r_k$ and that there exists $n^*$ such that for all $n \geq n^*$ 
\[
\int_{[0,1]^d} | \Phi_{k,n}(j) | \, dj < \frac{\epsilon}{2 |X|}\text{ and }\suppp \Phi_{k,n} \cap \suppp \Phi_{k',n} = \emptyset \text{ for } k \neq k'. 
\]
Now define
\[
\Phi_{\epsilon, X} \coloneqq \Phi_{\epsilon} + \sum_{k \in X} \Phi_{k,n^*}. 
\]
Finally, we define 
\[
\mathcal{L}(X, (y_i)_{i \in [0,1]^d}) \coloneqq \Phi_{\epsilon, X}. 
\]

Clearly, $\mathcal{L}(X, (y_i)_{i \in [0,1]^d})$ is a neural network. Moreover, by construction, if $X$ is finite and $k \in X$,
\[
\mathcal{L}\left(X, \left(y_i^X\right)_{i \in [0,1]^d}\right)(k) =  \Phi_{\epsilon}(k) + \Phi_{k,n^*}(k) = \lambda\left(\left(y_i^X\right)_{i \in [0,1]^d}\right)_k
\]
and $\lambda\left(((y_i^X)_{i \in [0,1]^d}\right) \sim (y_i^X)_{i \in [0,1]^d} \sim (y_i)_{i \in [0,1]^d}$. Therefore, the set on which $\lambda\left(y_i^X)_{i \in [0,1]^d}\right)_k$ and $(y_i)_{i \in [0,1]^d}$ differ is finite and thus the probability of it intersecting $X$ is zero if $X$ is a finite random set. This concludes the first part of the theorem.

For the second part observe that if $f$ is integrable, then since $g=f$ except on a finite set of measure zero, also $g$ is integrable and thus by the triangle inequality
\begin{align*}
\int_{[0,1]^d} \left|\mathcal{L}\left(X, \left(y_i^X\right)_{i \in [0,1]^d} \right) (j) - f(j)\right| \, dj 
&\leq \int_{[0,1]^d} \left|\Phi_{\epsilon}(j) - f(j)\right| \, dj + \sum_{k \in X} \int_{[0,1]^d} \left|  \Phi_{k,n^*} (j) \right| \, dj < \epsilon.
\end{align*}
This completes the proof.
\end{proof}

\section{Discussion}

Theorem \ref{thm:main} implies that deep neural networks are, in principle, able to predict data, that is completely uncorrelated with the training data. This observation might appear unintuitive, stunning, and maybe unbelievable at first, but the same can be said for most other deep learning results. We shall address a number of limitations and observations about Theorem \ref{thm:main} below. 

In a sense, the algorithm requires an uncountable number of data points. Such a situation is often excluded in applications despite the current trend of big data. It is, however, expected that the amount of data that we need to handle will increase with time. Following this reasoning, scholars have speculated that we might soon enter the era of \emph{very big data}. 

Many machine learning algorithms suffer from overfitting, i.e., their predictions are overly adapted to the training data, and they do not generalise to unseen data.
Equation \eqref{eq:PerfectAccuracyOnTestData} shows that the neural network produced by $\mathcal{L}$ predicts the labels of the points in $X$ correctly with probability $1$. However, on the training set, the accuracy is a bit worse as demonstrated by Equation \eqref{eq:goodAccuracyOnTrainingData}. In other words, $\mathcal{L}$ is not overfitting, but deliberately underfitting while the prediction remains accurate.

It has often been conjectured that deep learning techniques can remove the curse of dimensionality that plagues many high-dimensional approximation tasks. Careful inspection of the statement of Theorem \ref{thm:main} indeed reveals that the quality of prediction appears to be independent of the ambient dimension $d$. Another piece of conventional wisdom is that deep networks are more efficient than their shallow counterparts, \cite{Telgarsky2016BenefitsOfDepth, pmlr-v70-safran17a}. The construction of \eqref{eq:Phin} requires at least two hidden layers. Thus, the network is technically not a shallow network. It is fair say that this observation gives yet another illustration of the power of depth.

Finally, we mention that the construction is inspired by a prediction strategy for ordered data, presented in \cite{hardin2008peculiar}. The discussion in that paper indicates that the same result is most likely not valid in classical Zermelo-Fraenkel set theory. It might not be wise to prominently suggest that the current deep learning revolution and the overwhelming empirical success prove that the axiom of choice is an empirically reasonable assumption, but we leave the reader to her own conclusion/s. 

\section*{Acknowledgements}
All authors except for Philipp Grohs are indebted to Philipp Grohs for giving an especially lovely lecture at the Oberwolfach Seminar "Deep Learning" in October of 2018. This lecture led the participants to profound insights and culminated in this article and hopefully many more.
The authors are grateful for the hospitality of the Mathematisches Forschungsinstitut Oberwolfach which hosted this fantastic event.

\section*{Disclaimer}
We concede that this manuscript has mostly humouristic value. However, while the main theorem is based on a dubious application of the axiom of choice, it is a correct mathematical statement. Therefore, this manuscript at least highlights the dangers of applying mathematical theory to real-world applications blindly. 

\small
\bibliographystyle{abbrv}
\bibliography{references}

\begin{thebibliography}{10}

\bibitem{alfke2018nfft}
D.~Alfke, D.~Potts, M.~Stoll, and T.~Volkmer.
\newblock {NFFT} meets {K}rylov methods: Fast matrix-vector products for the
  graph {L}aplacian of fully connected networks.
\newblock {\em arXiv preprint arXiv:1808.04580}, 2018.

\bibitem{Barron1993}
A.~Barron.
\newblock Universal approximation bounds for superpositions of a sigmoidal
  function.
\newblock {\em IEEE Trans. Inf. Theory}, 39(3):930--945, 1993.

\bibitem{blechschmidt2018improving}
J.~Blechschmidt and R.~Herzog.
\newblock Improving policies for hamilton-jacobi-bellman equations by
  postprocessing.
\newblock 2018.

\bibitem{brown2008famine}
M.~E. Brown.
\newblock {\em Famine early warning systems and remote sensing data}.
\newblock Springer Science \& Business Media, 2008.

\bibitem{bruna2013invariant}
J.~Bruna and S.~Mallat.
\newblock Invariant scattering convolution networks.
\newblock {\em IEEE Trans. Pattern Anal. Mach. Intell.}, 35(8):1872--1886,
  2013.

\bibitem{Cybenko1989}
G.~Cybenko.
\newblock {Approximation by superpositions of a sigmoidal function}.
\newblock {\em Math. Control Signal}, 2(4):303--314, 1989.

\bibitem{dahl2012context}
G.~E. Dahl, D.~Yu, L.~Deng, and A.~Acero.
\newblock Context-dependent pre-trained deep neural networks for
  large-vocabulary speech recognition.
\newblock {\em IEEE Audio, Speech, Language Process.}, 20(1):30--42, 2012.

\bibitem{del2016numerical}
M.~Del~Razo and R.~LeVeque.
\newblock Numerical methods for interface coupling of compressible and almost
  incompressible fluids.
\newblock {\em SIAM J. Sci. Comput., submitted}, 2016.

\bibitem{drory2014semi}
A.~Drory, C.~Haubold, S.~Avidan, and F.~A. Hamprecht.
\newblock Semi-global matching: a principled derivation in terms of message
  passing.
\newblock In {\em German Conference on Pattern Recognition}, pages 43--53.
  Springer, 2014.

\bibitem{IMFReport}
I.~M. Fund.
\newblock Global financial stability report: A bumpy road ahead.
\newblock 2018.
\newblock Washington, DC, April.

\bibitem{glas2017two}
S.~Glas, A.~Mayerhofer, and K.~Urban.
\newblock Two ways to treat time in reduced basis methods.
\newblock In {\em Model Reduction of Parametrized Systems}, pages 1--16.
  Springer, 2017.

\bibitem{Goodfellow-et-al-2016}
I.~Goodfellow, Y.~Bengio, and A.~Courville.
\newblock {\em Deep Learning}.
\newblock MIT Press, 2016.

\bibitem{grohs2014parabolic}
P.~Grohs and G.~Kutyniok.
\newblock Parabolic molecules.
\newblock {\em Foundations of Computational Mathematics}, 14(2):299--337, 2014.

\bibitem{hansen2018stable}
A.~Hansen and L.~Thesing.
\newblock On the stable sampling rate for binary measurements and wavelet
  reconstruction.
\newblock {\em Applied and Computational Harmonic Analysis}, 2018.

\bibitem{hardin2008peculiar}
C.~S. Hardin and A.~D. Taylor.
\newblock A peculiar connection between the axiom of choice and predicting the
  future.
\newblock {\em The American Mathematical Monthly}, 115(2):91--96, 2008.

\bibitem{hinton2012deep}
G.~Hinton, L.~Deng, D.~Yu, G.~E. Dahl, A.-r. Mohamed, N.~Jaitly, A.~Senior,
  V.~Vanhoucke, P.~Nguyen, T.~N. Sainath, et~al.
\newblock Deep neural networks for acoustic modeling in speech recognition: The
  shared views of four research groups.
\newblock {\em IEEE Signal Process. Mag.}, 29(6):82--97, 2012.

\bibitem{hinz2018framework}
P.~Hinz and S.~van~de Geer.
\newblock A framework for the construction of upper bounds on the number of
  affine linear regions of {ReLU} feed-forward neural networks.
\newblock {\em arXiv preprint arXiv:1806.01918}, 2018.

\bibitem{HORNIK1991251}
K.~Hornik.
\newblock Approximation capabilities of multilayer feedforward networks.
\newblock {\em Neural Networks}, 4(2):251–257, 1991.

\bibitem{Hornik1989universalApprox}
K.~Hornik, M.~Stinchcombe, and H.~White.
\newblock Multilayer feedforward networks are universal approximators.
\newblock {\em Neural Netw.}, 2(5):359--366, 1989.

\bibitem{Huang2017DenselyCC}
G.~Huang, Z.~Liu, L.~van~der Maaten, and K.~Q. Weinberger.
\newblock Densely connected convolutional networks.
\newblock {\em 2017 IEEE Conference on Computer Vision and Pattern Recognition
  (CVPR)}, pages 2261--2269, 2017.

\bibitem{kawaguchi2016deep}
K.~Kawaguchi.
\newblock Deep learning without poor local minima.
\newblock In {\em Adv. Neural Inf. Process. Syst.}, pages 586--594, 2016.

\bibitem{Krizhevsky2012Imagenet}
A.~Krizhevsky, I.~Sutskever, and G.~Hinton.
\newblock Image{N}et classification with deep convolutional neural networks.
\newblock In {\em Advances in Neural Information Processing Systems 25}, pages
  1097--1105. Curran Associates, Inc., 2012.

\bibitem{LeCun2015DeepLearning}
Y.~LeCun, Y.~Bengio, and G.~Hinton.
\newblock {Deep learning}.
\newblock {\em Nature}, 521(7553):436--444, 2015.

\bibitem{leshno1993multilayer}
M.~Leshno, V.~Y. Lin, A.~Pinkus, and S.~Schocken.
\newblock Multilayer feedforward networks with a nonpolynomial activation
  function can approximate any function.
\newblock {\em Neural networks}, 6(6):861--867, 1993.

\bibitem{lunz2018adversarial}
S.~Lunz, O.~{\"O}ktem, and C.-B. Sch{\"o}nlieb.
\newblock Adversarial regularizers in inverse problems.
\newblock {\em arXiv preprint arXiv:1805.11572}, 2018.

\bibitem{macdonald2018improved}
J.~Macdonald and L.~Ruthotto.
\newblock Improved susceptibility artifact correction of echo-planar mri using
  the alternating direction method of multipliers.
\newblock {\em Journal of Mathematical Imaging and Vision}, 60(2):268--282,
  2018.

\bibitem{marx1867kapital}
K.~Marx.
\newblock Das {K}apital: {K}ritik der politischen {\"o}konomie.
\newblock {\em Germany: Verlag von Otto Meisner}, 1885:1894, 1867.

\bibitem{MP43}
W.~McCulloch and W.~Pitts.
\newblock A logical calculus of ideas immanent in nervous activity.
\newblock {\em Bull. Math. Biophys.}, 5:115--133, 1943.

\bibitem{Mhaskar1993}
H.~Mhaskar.
\newblock Approximation properties of a multilayered feedforward artificial
  neural network.
\newblock {\em Advances in Computational Mathematics}, 1(1):61–80, Feb 1993.

\bibitem{mhaskar-micchelli}
H.~N. Mhaskar and C.~Micchelli.
\newblock {Approximation by superposition of sigmoidal and radial basis
  functions}.
\newblock {\em Adv. Appl. Math.}, 13:350--373, 1992.

\bibitem{monshi2016burden}
B.~Monshi, M.~Vujic, D.~Kivaranovic, A.~Sesti, W.~Oberaigner, I.~Vujic,
  S.~Ortiz-Urda, C.~Posch, H.~Feichtinger, M.~Hackl, et~al.
\newblock The burden of malignant melanoma--lessons to be learned from
  {A}ustria.
\newblock {\em European Journal of Cancer}, 56:45--53, 2016.

\bibitem{neufeld2013superreplication}
A.~Neufeld, M.~Nutz, et~al.
\newblock Superreplication under volatility uncertainty for measurable claims.
\newblock {\em Electronic journal of probability}, 18, 2013.

\bibitem{perekrestenko2018universal}
D.~Perekrestenko, P.~Grohs, D.~Elbr{\"a}chter, and H.~B{\"o}lcskei.
\newblock The universal approximation power of finite-width deep {ReLU}
  networks.
\newblock {\em arXiv preprint arXiv:1806.01528}, 2018.

\bibitem{petersen2018optimal}
P.~Petersen and F.~Voigtlaender.
\newblock Optimal approximation of piecewise smooth functions using deep {ReLU}
  neural networks.
\newblock {\em Neural Networks}, 108:296--330, 2018.

\bibitem{reisenhofer2018haar}
R.~Reisenhofer, S.~Bosse, G.~Kutyniok, and T.~Wiegand.
\newblock A haar wavelet-based perceptual similarity index for image quality
  assessment.
\newblock {\em Signal Processing: Image Communication}, 61:33--43, 2018.

\bibitem{pmlr-v70-safran17a}
I.~Safran and O.~Shamir.
\newblock Depth-width tradeoffs in approximating natural functions with neural
  networks.
\newblock In {\em Proceedings of the 34th International Conference on Machine
  Learning}, volume~70 of {\em Proceedings of Machine Learning Research}, pages
  2979--2987, 2017.

\bibitem{saito2017underwater}
N.~Saito and D.~S. Weber.
\newblock Underwater object classification using scattering transform of sonar
  signals.
\newblock In {\em Wavelets and Sparsity XVII}, volume 10394, page 103940K.
  International Society for Optics and Photonics, 2017.

\bibitem{schmidhuber2015deep}
J.~Schmidhuber.
\newblock Deep learning in neural networks: An overview.
\newblock {\em Neural networks}, 61:85--117, 2015.

\bibitem{ShaCC2015provableAppDNN}
U.~Shaham, A.~Cloninger, and R.~Coifman.
\newblock Provable approximation properties for deep neural networks.
\newblock {\em Appl. Comput. Harmon. Anal.}
\newblock to appear.

\bibitem{silver2017mastering}
D.~Silver, T.~Hubert, J.~Schrittwieser, I.~Antonoglou, M.~Lai, A.~Guez,
  M.~Lanctot, L.~Sifre, D.~Kumaran, T.~Graepel, et~al.
\newblock Mastering chess and shogi by self-play with a general reinforcement
  learning algorithm.
\newblock {\em arXiv preprint arXiv:1712.01815}, 2017.

\bibitem{simonyan2014very}
K.~Simonyan and A.~Zisserman.
\newblock Very deep convolutional networks for large-scale image recognition.
\newblock {\em arXiv preprint arXiv:1409.1556}, 2014.

\bibitem{Telgarsky2016BenefitsOfDepth}
M.~Telgarsky.
\newblock Benefits of depth in neural networks.
\newblock arXiv:1602.04485.

\bibitem{usunier2016episodic}
N.~Usunier, G.~Synnaeve, Z.~Lin, and S.~Chintala.
\newblock Episodic exploration for deep deterministic policies: An application
  to {S}tar{C}raft micromanagement tasks.
\newblock {\em arXiv preprint arXiv:1609.02993}, 2016.

\bibitem{vidal2017mathematics}
R.~Vidal, J.~Bruna, R.~Giryes, and S.~Soatto.
\newblock Mathematics of deep learning.
\newblock {\em arXiv preprint arXiv:1712.04741}, 2017.

\bibitem{weber2017coarse}
M.~Weber, E.~Saucan, and J.~Jost.
\newblock Coarse geometry of evolving networks.
\newblock {\em Journal of Complex Networks}, 6(5):706--732, 2017.

\bibitem{wiatowski2018mathematical}
T.~Wiatowski and H.~B{\"o}lcskei.
\newblock A mathematical theory of deep convolutional neural networks for
  feature extraction.
\newblock {\em IEEE Trans. Inf. Theory}, 64(3):1845--1866, 2018.

\bibitem{wu2016stimulated}
C.~Wu, P.~Karanasou, M.~J. Gales, and K.~C. Sim.
\newblock Stimulated deep neural network for speech recognition.
\newblock Technical report, University of Cambridge, 2016.

\bibitem{yannakakis2017artificial}
G.~N. Yannakakis and J.~Togelius.
\newblock {\em Artificial Intelligence and Games}.
\newblock Springer, 2017.

\bibitem{YAROTSKY2017103}
D.~Yarotsky.
\newblock Error bounds for approximations with deep {ReLU} networks.
\newblock {\em Neural Networks}, 94:103 -- 114, 2017.

\end{thebibliography}

\end{document}